\documentclass[11pt,letterpaper]{jams-l}

\usepackage[margin=1in]{geometry}

\usepackage[utf8]{inputenc} 
\usepackage[T1]{fontenc}    
\usepackage{hyperref}       
\usepackage{url}            
\usepackage{booktabs}       
\usepackage{amsfonts}       
\usepackage{amsmath}        
\usepackage{amsthm}
\usepackage{nicefrac}       
\usepackage{microtype}      
\usepackage{xcolor}         
\usepackage{cite}

\usepackage{graphicx}
\usepackage{thmtools}

\title[Better neural network expressivity]{Better neural network expressivity: \\ subdividing the simplex}

\author[Bakaev]{Egor Bakaev}
\address{Department of Computer Science, University of Copenhagen}
\email{egor.bakaev@gmail.com}

\author[Brunck]{Florestan Brunck}
\address{Department of Computer Science, University of Copenhagen}
\email{flbr@di.ku.dk}

\author[Hertrich]{Christoph Hertrich}
\address{University of Technology Nuremberg}
\email{christoph.hertrich@utn.de}

\author[Stade]{Jack Stade}
\address{Department of Computer Science, University of Copenhagen}
\email{jast@di.ku.dk}

\author[Yehudayoff]{Amir Yehudayoff}
\address{Department of Computer Science, University of Copenhagen, and Department of Mathematics, Technion--IIT}
\email{amir.yehudayoff@gmail.com}

\usepackage[noabbrev,capitalize,nameinlink]{cleveref}

\newcommand{\ip}[2]{\langle #1,#2 \rangle}
\newcommand{\conv}{\mathsf{conv}}

\newcommand{\R}{\mathbb{R}}

\DeclareMathOperator{\relu}{\mathsf{ReLU}}
\DeclareMathOperator{\ICNN}{\mathsf{ICNN}}
\DeclareMathOperator{\CPWL}{\mathsf{CPWL}}

\newcommand{\MAX}{\mathsf{MAX}}
\DeclareMathOperator{\cP}{\mathcal{P}}

\newtheorem{theorem}{Theorem}
\newtheorem{lemma}[theorem]{Lemma}
\newtheorem{claim}[theorem]{Claim}
\newtheorem{prop}[theorem]{Proposition}

\newtheorem{corollary*}{Corollary}

\newtheorem{remark}{Remark}

\begin{document}

\renewcommand{\max}{\mathsf{max}}

\begin{abstract}
This work studies the expressivity of $\relu$ neural networks with a focus on their depth. A sequence of previous works showed that \mbox{$\lceil \log_2(n+1) \rceil$} hidden layers are sufficient to compute
all continuous piecewise linear ($\CPWL$) functions on~$\R^n$. Hertrich, Basu, Di Summa, and Skutella (NeurIPS\,'21 / SIDMA\,'23) conjectured that this result is optimal in the sense that there are $\CPWL$ functions on $\R^n$,
like the maximum function, that require this depth. We disprove the conjecture and show that $\lceil\log_3(n-1)\rceil+1$ hidden layers are sufficient to compute all $\CPWL$ functions on $\R^n$.

A key step in the proof is that $\relu$ neural networks with two hidden layers can exactly represent the maximum function of five inputs. More generally, we show that $\lceil\log_3(n-2)\rceil+1$ hidden layers are sufficient to compute the maximum of $n\geq 4$ numbers.
Our constructions almost match the $\lceil\log_3(n)\rceil$ lower bound of Averkov, Hojny, and Merkert (ICLR\,'25) in the special case of $\relu$ networks with weights that are decimal fractions.
The constructions have a geometric interpretation via polyhedral subdivisions of the simplex into ``easier'' polytopes.
\end{abstract}

\maketitle

\section{Introduction}

Neural networks with rectified linear unit ($\relu$) activations are among the most common models in modern machine learning. A \emph{$\relu$ network} with \emph{depth} $k+1$ is defined via $k+1$ affine transformations $T^{(i)}:\R^{n_{i-1}}\to \R^{n_i}$ for $i=1,\dots,k+1$. It \emph{computes} the function defined as alternating composition
\[
T^{(k+1)}\circ \relu \circ T^{(k)} \circ \cdots \circ  \relu \circ T^{(2)} \circ \relu \circ{} T^{(1)}
\]
of the affine transformations with component-wise applications of the $\relu$ function $\relu(x) = \max \{0,x\}$. We say that the network has $k$ \emph{hidden layers}. Alternatively, $\relu$ networks can be defined as directed acyclic graphs of neurons, each of which computes a function $x\mapsto\relu \big(b + \sum_i a_i x_i \big)$.

Functions represented by $\relu$ networks are \emph{continuous and piecewise linear} ($\CPWL$), and every such function on $\R^n$ can be computed by a $\relu$ network~\cite{arora2018understanding}.
This naturally leads to complexity-theoretic questions, such as: what are the minimal size (i.e., number of neurons) and depth needed to represent specific $\CPWL$ functions? Such questions have been extensively studied recently (see~\cite{arora2018understanding,williams2018limits,hertrich2021towards,haase2023lower,valerdi2024minimal,huchette2023deep,hertrich2024neural} and references within), but many major questions are still open, such as the minimum number of layers required to represent all $\CPWL$ functions, which is the topic of this paper.

Denote by $\CPWL_n$ the space of $\CPWL$ functions $f : \R^n \to \R$ and by $\relu_{n,k}$ the subset of $\CPWL_n$ representable with $k$ hidden layers.
An important $\CPWL$ function for understanding neural network depth is the $\MAX_n$ function defined by
\[\MAX_n(x) = \MAX_n(x_1,\ldots,x_n) = \max \{x_1,\ldots,x_n\}.\]
Wang and Sun~\cite{wang2005generalization} showed that every function in $\CPWL_n$
can be written as a linear combination
of $\MAX_{n+1}$ functions applied to some  affine functions:
for every $f \in \CPWL_n$, there are affine maps
$A_1,\ldots,A_s$ from $\R^n$ to $\R^{n+1}$ 
and $\sigma_1,\ldots,\sigma_s \in \{\pm 1\}$ so that
\begin{equation}\label{eq:wangsun}f(x) = \sum_{i = 1}^s \sigma_i \MAX_{n+1}(A_i(x)).\end{equation}
As observed by Arora, Basu, Mianjy, and Mukherjee\cite{arora2018understanding}, this fact implies that 
\[\CPWL_n \subseteq \relu_{n,\lceil \log_2 (n+1) \rceil}.\]
In words, every $\CPWL$ function on $\R^n$ can be represented with $\lceil \log_2 (n+1) \rceil$ hidden layers. The reason is that pairwise maxima can be represented with one hidden layer via $\MAX_2(x_1,x_2) = \relu(x_2-x_1) + x_1$. Then, $\MAX_{n+1}$---and therefore also the representation \eqref{eq:wangsun}---can be realized with $\lceil \log_2 (n+1) \rceil$ hidden layers by implementing it in a binary tree manner.

From this discussion, it follows that the depth complexity of $\CPWL_n$ is essentially equal to the depth complexity of $\MAX_{n+1}$.\footnote{$\MAX_{n+1}$ is defined on $\R^{n+1}$ rather than $\R^n$, but this can be fixed by setting one coordinate to zero and considering $\max\{0,x_1,\dots,x_n\}$ instead, which has the same depth complexity as $\MAX_{n+1}$, see \cite{hertrich2021towards}. We nevertheless stick to our definition for symmetry reasons.}
This highlights the importance of $\MAX_n$, and motivates the study of its depth complexity.

Summarizing, every $\CPWL$ function on $\R^n$ can be realized with at most $\lceil \log_2 (n+1) \rceil$ hidden layers, leading to the natural question of whether this upper bound is sharp. In fact, to the best of our knowledge, the best lower bound is $2$, meaning that there is no function known that provably requires more than two hidden layers.

Hertrich, Basu, Di Summa, and Skutella~\cite{hertrich2021towards} conjectured that the $\lceil \log_2 (n) \rceil$ upper bound
on the required number of hidden layers for $\MAX_n$ is sharp. 
This conjecture led to a sequence of works that proved it in a variety of special cases. 
Mukherjee and Basu~\cite{mukherjee2017lower} are the first to prove the conjecture for the case $n=3$, i.e., showing that two hidden layers are necessary for $\MAX_3$.
Haase, Hertrich, and Loho~\cite{haase2023lower} proved the conjecture for general $n$ when the $\relu$ network is allowed to use only integer weights.
The techniques from~\cite{haase2023lower} were extended by Averkov, Hojny, and Merkert~\cite{averkov2025expressiveness}, who proved logarithmic lower bounds for the case of fractional weights whose denominators do not have some fixed prime factor.
The results of Valerdi~\cite{valerdi2024minimal} imply that the conjecture holds for the input convex neural networks ($\ICNN$s) model, which was defined by Amos, Xu and Kolter~\cite{amos2017input} and studied in subsequent works (e.g.~\cite{bunning2021input,chen2018optimal,chen2020data,bunne2023learning,hertrich2024neural}). Stronger lower bounds for $\ICNN$s and monotone $\relu$ networks were shown by Bakaev, Brunck, Hertrich, Reichman, and Yehudayoff~\cite{bakaev2025depth}.
Grillo, Hertrich, and Loho proved an $\Omega( \log \log n)$ lower bound for $\relu$ networks that ``respect the singularities of $\MAX_n$''
(see \cite{grillo2025depth} for a formal definition).
Their proof also implies the conjecture for $n=5$ for networks that ``respect the singularities of $\MAX_n$'', a statement that was already shown in a computer-aided manner in \cite{hertrich2021towards}.

\subsection{Our Contribution.}
Our main result is that the conjecture is false. More precisely, we prove that the maximum of $n$ numbers can be computed with $\lceil\log_3(n-2)\rceil+1$ hidden layers.

\begin{theorem}\label{thm:maxtothe3}
    For $n\geq 1$, we have $\MAX_{3^n+2}\in\relu_{n+1}$.
\end{theorem}

By the discussion above, this implies that
every $\CPWL$ function defined on $\R^n$ can be represented with $\lceil\log_3(n-1)\rceil+1$ hidden layers.

\begin{theorem}\label{thm:main}
    For $n \geq 3$, we have $\CPWL_n=\relu_{n,\lceil\log_3(n-1)\rceil+1}$.
\end{theorem}

Compared to the previous upper bound~\cite{arora2018understanding}, we improve the asymptotically required depth by a constant factor of $\ln(2)/\ln(3)\approx0.63$. 
This surprising result shows that computing the maximum of $n$ numbers can be done in faster parallel time than $\log_2(n)$ even when all ``max operations'' have fan-in two. 
The summation operation allows to boost the parallel running time. 

Naturally, our construction does not fall into any of the special cases under which the conjecture has been proven before. However, interestingly, as we only use weights with denominators that are a power of two, our construction falls into the setting of Averkov, Hojny, and Merkert~\cite{averkov2025expressiveness}, who proved a $\lceil\log_3(n)\rceil$ lower bound on the number of hidden layers required for $\MAX_n$ in the special case of $\relu$ networks with weights that are decimal fractions. We almost match this lower bound (up to a difference of at most one layer).

It is noteworthy that \Cref{thm:maxtothe3} already improves the depth complexity of $\MAX_5$. 

\begin{prop}
\label{prop:max5}
The minimum number of hidden layers needed for computing $\MAX_5$ is exactly two.
\end{prop}

 This also implies that every function in $\CPWL_4$ can be implemented with two hidden layers, sharpening the contrast to the $\ICNN$ model, for which there are functions in $\CPWL_4$ that need arbitrary large depth (see \cite{valerdi2024minimal} for the proof using neighborly polytopes and \cite{bakaev2025depth} for a more detailed discussion).

The proof of our result is short and elementary, we provide it in \Cref{sec:proof}. However, it hides a beautiful geometric intuition that guided us in developing the proof and that might also be useful to further improve the current upper or lower bounds. We therefore explain this intuition in \Cref{sec:geometry}. The geometric intuition is based on relations between neural networks and polyhedral and tropical geometry, which were already studied in several previous works, see e.g.,~\cite{zhang2018tropical,maclagan2021introduction,maragos2021tropical,hertrich2021towards,huchette2023deep,haase2023lower,brandenburg2025decomposition,hertrich2024neural,valerdi2024minimal}.

We hope that our result initiates follow-up work on finding the true required depth to represent all $\CPWL$ functions. There is a gap to close between $2$ and $\lceil\log_3(n-1)\rceil+1$ hidden layers. While we resolve the previously open question of whether $\MAX_5$ can be represented with two hidden layers, it intriguingly remains open whether the same is possible for $\MAX_6$.

We mention a similarity to the computational complexity of multiplying two $n \times n$ matrices. The trivial algorithm runs in time $O(n^3)$.
In his seminal work, Strassen showed that multiplying two $2 \times 2$ matrices can be done with seven multiplications instead of eight~\cite{strassen1969gaussian}. This finite-dimensional improvement leads to a better algorithm with running time $O(n^{\log_2 7}) \approx O(n^{2.8})$. A long line of works (see~\cite{cohn2003group,coppersmith1982rapid,gall2018improved
} and the many references within) found better and better upper bounds on the  matrix multiplication exponent $\omega$
and the current best upper bound is $\omega \leq 2.371552$~\cite{alman2025more}.
Coming back to the maximum function, thinking of the depth parameter $d$ as fixed, we can trivially compute the maximum of $2^d$ numbers with $d$ hidden layers. The finite-dimensional \Cref{prop:max5} leads to a polynomial improvement. With $d$ hidden layers, we can compute the maximum of $\approx (5/2)^d$ numbers. With a more technical analysis this can be boosted to $\approx 3^d$. 
For the matrix multiplication exponent, we have the trivial lower bound $\omega \geq 2$. On the other hand, to the best of our knowledge it could be the case that two hidden layers suffice for computing the maximum of $n$ numbers for all $n$. 

\subsection{Universal approximation and required size.}

Our paper is concerned with the study of exact expressivity. The celebrated \emph{universal approximation theorems} (see~\cite{cybenko1989approximation,hornik1989multilayer} for initial versions and \cite{leshno1993multilayer} for a version that encompasses $\relu$) state that every continuous function on a bounded domain can be approximated with one hidden layer networks. Related to this, Safran, Reichman, and Valiant~\cite{safran2024many} studied how to efficiently approximate $\MAX_n$. Nevertheless, for a rigorous mathematical understanding of neural networks, it is inevitable to study the precise set of functions representable by different network architectures. Also, universal approximation theorems only work on bounded domains and might require a large number of neurons. In contrast, it is easy to verify that our construction needs only linearly many neurons to compute $\MAX_n$. It is, however, more tricky to determine the required size to represent an arbitrary $\CPWL$ function with our construction, as this depends on the number $s$ of terms in the representation \eqref{eq:wangsun}. The best known upper bound on $s$ was investigated in a series of works~\cite{he2018relu,hertrich2021towards,chen2022improved,brandenburg2025decomposition} and is $s\in \mathcal{O}(p^{n+1})$, where $p$ is the number of different affine functions represented locally by the given $\CPWL$ function. Another related stream of research is how one can trade depth and size against each other, with several results implying that mild increases in depth often allow significantly smaller neural networks~\cite{montufar14number,telgarsky2016benefits,eldan2016power,arora2018understanding,ergen24topology,safran2022depth}.

\section{Direct proofs}\label{sec:proof}

\subsection{Computing the maximum of five numbers with two hidden layers}
\label{sec-max5}

In this section we give an explicit representation of $\MAX_5$ with two hidden layers. To this end, let $M : \R^5 \to \R$ be defined by 
\begin{align}
\label{eq-m}
M(x) & = 0.5 \big( P_1(x)+P_2(x)+P_3(x)+P_4(x)+Q(x) \\
\notag & \qquad \qquad -R_{13}(x)-R_{14}(x)-R_{23}(x)-R_{24}(x) \big),
\end{align}
where
\begin{align*}
	P_1(x) &= \max ( \max (2x_5, x_1+x_2 ),   \max (x_1,x_3)+\max(x_1,x_4) ) 
\\ &= \max (2x_5,  x_1+x_2, 2x_1, x_1+x_3, x_1+x_4, x_3+x_4),
\displaybreak[1]
\\[1ex] 
 P_2(x) &= \max ( \max (2x_5,  x_1+x_2 ),   \max (x_2,x_3)+\max(x_2,x_4) )
\\ &= \max (2x_5,  x_1+x_2, 2x_2, x_2+x_3, x_2+x_4, x_3+x_4),
\displaybreak[1]
\\[1ex] 
 P_3(x) &= \max (  \max (2x_5,  x_3+x_4 ),   \max (x_3,x_1)+\max(x_3,x_2) )
\\ &= \max (2x_5,  x_3+x_4, 2x_3, x_3+x_1, x_3+x_2, x_1+x_2),
\displaybreak[1]
\\[1ex] 
 P_4(x) &= \max ( \max (2x_5,  x_3+x_4 ),   \max (x_4,x_1)+\max(x_4,x_2) )
\\ &= \max (2x_5,  x_3+x_4, 2x_4, x_4+x_1, x_4+x_2, x_1+x_2),
\displaybreak[1]
\\[1ex] 
Q(x)  &= \max (2x_5, \max (x_1+x_2, x_3+x_4 ) )
\\ &= \max (2x_5,  x_1+x_2, x_3 + x_4),
\displaybreak[1]
\\[1ex] 
R_{13}(x)  &= \max(\max(2x_5, x_1+x_3), \max(x_1+x_2, x_3+x_4) )
\\ &= \max (2x_5,  x_1+x_3, x_1+x_2, x_3+x_4 ),
\displaybreak[1]
\\[1ex]  
R_{14}(x)  &= \max(\max(2x_5, x_1+x_4), \max(x_1+x_2, x_3+x_4) )
\\ &= \max (2x_5,  x_1+x_4, x_1+x_2, x_3+x_4 ),
\displaybreak[1]
\\[1ex]  
R_{23}(x)  &=  \max(\max(2x_5, x_2+x_3), \max(x_1+x_2, x_3+x_4) )  
\\ &= \max (2x_5,  x_2+x_3, x_1+x_2, x_3+x_4 ),
\displaybreak[1]
\\[1ex] 
R_{24}(x)  &=  \max(\max(2x_5, x_2+x_4), \max(x_1+x_2, x_3+x_4) )  
\\ &= \max (2x_5,  x_2+x_4, x_1+x_2, x_3+x_4 ).
\end{align*}

\begin{claim}\label{claim:max5rep}
$\MAX_5 = M$.
\end{claim} 

\begin{proof}
We use a simple case analysis.
If $\MAX_5(x) = x_5$ then $2x_5$ is not smaller than the sum of any pair of coordinates $x_1, x_2, x_3, x_4, x_5$, and then each of the nine terms defined above equals to $2x_5$, therefore $M(x_1,x_2,x_3,x_4,x_5) = x_5$. So now we can assume that $\max(x_1,x_2,x_3,x_4,x_5)$ is not $x_5$.

The expression $M(x)$ is symmetric with respect to the following operations: swapping indices $1$~and~$2$; swapping indices $3$ and $4$; simultaneously swapping the pairs
of indices $(1,2)$ and $(3,4)$. 
This is because with these swaps of indices some of the terms just swap with each other:
\begin{align*}
\text{if } 1 \leftrightarrow 2, \text{  then }&
\,&P_1&\leftrightarrow\!\!\!\!&&P_2,
\,&R_{13}&\leftrightarrow\!\!\!\!&&R_{23},
\,&R_{14}&\leftrightarrow\!\!\!\!&&R_{24}, &&\text{ other terms stay the same;}
\\ \text{if } 3 \leftrightarrow 4, \text{  then }&
\,&P_3&\leftrightarrow\!\!\!\!&&P_4,
\,&R_{13}&\leftrightarrow\!\!\!\!&&R_{14},
\,&R_{23}&\leftrightarrow\!\!\!\!&&R_{24}, &&\text{ other terms stay the same;}
\\ \text{if } (1,2) \leftrightarrow (3,4), \text{  then }&
\,&P_1&\leftrightarrow\!\!\!\!&&P_3,
\,&P_2&\leftrightarrow\!\!\!\!&&P_4,
\,&R_{14}&\leftrightarrow\!\!\!\!&&R_{23}, &&\text{ other terms stay the same.}
\end{align*}

So without loss of generality, we can suppose that
\[ x_1 =\max(x_1,x_2,x_3,x_4)=\max(x_1,x_2,x_3,x_4,x_5),\] 
and the other cases will follow by symmetry.
When $x_1$ is the largest input, we have: 

\begin{align*}
    x_1=&\ \MAX_5(x)
\!\!\!\!&&\Longrightarrow \ 
 P_1(x)\!\!\!\!\! &&=\max ( 2x_5,x_1+x_2,2x_1,x_1+x_3,x_1+x_4,x_3+x_4)
\\ & && &&= 2x_1,
\displaybreak[1] \\[1ex] 
x_2+x_1 &\geq 2x_2, && &&
\\  x_2+x_1 &\geq x_2+x_4 
\!\!\!\!&&\Longrightarrow \ 
 P_2(x)\!\!\!\!\! &&= \max (2x_5, x_1+x_2,2x_2,x_2+x_3,x_2+x_4,x_3+x_4)
\\ & && &&= \max (2x_5, x_1+x_2,x_2+x_3,x_3+x_4)
 \\ & && &&= R_{23}(x),
\displaybreak[1] \\[1ex] 
x_3+x_1 &\geq 2x_3, && &&
\\  x_3+x_1 &\geq x_3+x_2 
\!\!\!\!&&\Longrightarrow \ 
 P_3(x)\!\!\!\!\! &&= \max (2x_5,x_3+x_4,2x_3,x_3+x_1,x_3+x_2,x_1+x_2)
\\ & && &&= \max (2x_5,x_3+x_4,x_3+x_1,x_1+x_2)
\\ & && &&= R_{13}(x),
\displaybreak[1] \\[1ex] 
x_4+x_1 &\geq 2x_4, && &&
\\  x_4+x_1 &\geq x_4+x_2
\!\!\!\!&&\Longrightarrow \ 
 P_4(x)\!\!\!\!\! &&= \max  (2x_5,x_3+x_4,2x_4,x_4+x_1,x_4+x_2,x_1+x_2)
\\ & && &&= \max  (2x_5,x_3+x_4,x_4+x_1,x_1+x_2)
\\ & && &&= R_{14}(x),
\displaybreak[1] \\[1ex] 
x_1+x_2 &\geq x_2+x_4 
\!\!\!\!&&\Longrightarrow \ 
 R_{24}(x)\!\!\!\!\! &&= \max  (2x_5,x_2+x_4,x_1+x_2,x_3+x_4 )
\\  & && &&= \max  (2x_5,x_1+x_2,x_3+x_4 )
\\ & && &&= Q(x).
\end{align*}

We see that all but one of the terms in the sum cancel out:
\begin{align*}
	M(x) & = 0.5 \big( P_1(x)+P_2(x)+P_3(x)+P_4(x)+Q(x)
    -R_{13}(x)-R_{14}(x)-R_{23}(x)-R_{24}(x) \big)
	\\& = 0.5 P_1(x) 
    = x_1, 
\end{align*}

which is $\MAX_5(x)$, concluding the proof.
\end{proof}

With this representation of the maximum function, we can prove \Cref{prop:max5}.

\begin{proof}[Proof of \Cref{prop:max5}]
    As each of the nine terms in the definition of $M$ can be computed with two hidden layers, the same is true for $M$ itself, and therefore, by \Cref{claim:max5rep}, for $\MAX_5$.
\end{proof}

\subsection{Improved upper bound on the depth}\label{sec:newupperbound}

Computing $\MAX_5$ in $2$ hidden layers already allows us to compute $\MAX_n$ in $2\lceil\log_5(n)\rceil\approx 0.86\log_2(n)$ hidden layers. In fact, \Cref{claim:max5rep} can be pushed further. In this section, we show how to compute the maximum of $3^n+2$ numbers with $n+1$ hidden layers. This allows $\MAX_n$ to be computed in $\lceil\log_3(n-2)\rceil+1\approx 0.63\log_2(n)$ hidden layers.

Let $T_{a,b}$ be the function $\R^{4a+b} \to \R$ that equals the maximum of $a$ elements of type $\max(x_1, x_2) + \max(x_3, x_4)$, and $b$ extra arguments:
\begin{align*}
T_{a,b}(x)  = \max \big( & \max(x_1,x_2)+\max(x_3,x_4), \max(x_5,x_6)+\max(x_7,x_8), \ldots \\
& \ldots, \max(x_{4a-3}, x_{4a-2})+\max(x_{4a-1}, x_{4a}), x_{4a+1}, \ldots , x_{4a+b}\big) .    
\end{align*}

Let $\mathcal{T}_{a, b}$ be the vector space spanned by the functions of the form $T_{a, b}\circ L$ for linear transformations $L:\R^{4a+b}\rightarrow \R^{4a+b}$. 

\begin{claim}
	\label{claim-max5}
    Viewing $T_{a, b+4}$ as a function $\R^{4a+b+4+1}\rightarrow\R$ that ignores the last coordinate, we have that $T_{a, b+4}\in \mathcal{T}_{a+1, b+1}$.
\end{claim}
\begin{proof}
Write $T_{a, b+4}=\max(x_{i+1}, x_{i+2}, x_{i+3}, x_{i+4}, T_{a, b})$, where $i=4a+b$. By \Cref{eq-m}, this can be written as a linear combination of the nine functions $P_1, P_2, P_3, P_4, Q, R_{13}, R_{14}, R_{23}$ and $R_{24}$. Here, for example,
\begin{align*}
P_1(x)
& =\max\big(2T_{a, b}, x_{i+1}+x_{i+2}, \max(x_{i+1}, x_{i+3})+\max(x_{i+1}, x_{i+4})\big) \\
& =2 \max \left( T_{a, b}, \frac{x_{i+1}+x_{i+2}}{2}, \max \left(\frac{x_{i+1}}{2}, \frac{x_{i+3}}{2} \right)+\max\left(\frac{x_{i+1}}{2}, \frac{x_{i+4}}{2} \right) \right) , \\  
Q(x)  
 &= 2 \max \left( T_{a, b},  \frac{x_{i+1}+x_{i+2}}{2}, \frac{x_{i+3} + x_{i+4}}{2} \right) , \\
R_{13}(x)  &= 2 \max \left( T_{a, b},  \frac{x_{i+1}+x_{i+3}}{2}, \frac{x_{i+1}+x_{i+2}}{2}, \frac{x_{i+3}+x_{i+4}}{2} \right) .
\end{align*}
It follows that \begin{align*}
P_1(x) & \in \mathcal{T}_{a+1, b+1}  ,  \\
Q(x) &\in \mathcal{T}_{a, b+2}, 
\\ 
R_{13}(x) &\in \mathcal{T}_{a, b+3} . 
\end{align*}
Similarly, 
\begin{align*}
	P_2(x), P_3(x), P_4(x)  &\in \mathcal{T}_{a+1, b+1}, 
\\ R_{14}(x), R_{23}(x), R_{24}(x)  &\in \mathcal{T}_{a, b+3}.
\end{align*}
Since $\max (x_1, x_2)=\max (\max(x_1,x_2) + \max(0,0))$, we have that $\mathcal{T}_{a, b+2} \subseteq \mathcal{T}_{a, b+3} \subseteq \mathcal{T}_{a+1, b+1}$ (by ignoring the last coordinates). We can conclude that $T_{a, b+4}$ is in $\mathcal{T}_{a+1, b+1}$. 
\end{proof}

\begin{claim}\label{claim:compose-max}
	If the function $T_{3^{n-1},2}$ can be computed in depth $k$, then so can $T_{0,3^n+2}$.
\end{claim} 
\begin{proof}
Since $T_{a, b+4}\in \mathcal{T}_{a+1, b+1}$, we have $\mathcal{T}_{a, b+4}\subseteq \mathcal{T}_{a+1, b+1}$. Applying this $3^{n-1}$ times, 
\[\mathcal{T}_{0, 3^n+2} 
\subseteq \mathcal{T}_{1, 3^n-3+2}
\subseteq \mathcal{T}_{2, 3^n-6+2} 
\subseteq \mathcal{T}_{3, 3^n-9+2} 
\subseteq 
\cdots
\subseteq \mathcal{T}_{3^{n-1}, 2}.
\qedhere
\]
\end{proof}

\begin{claim}\label{claim:next-layer}
	If the function $T_{0,3^{n-1}+2}$ can be computed in depth $k-1$, then $T_{3^{n-1},2}$ can be computed in depth $k$.
\end{claim} 
\begin{proof}
 Compute $3^{n-1}$ elements of type $\max(x_1, x_2) + \max(x_3, x_4)$ in the first layer, then apply $T_{0,3^{n-1}+2}$ to them and two extra inputs.
\end{proof}

\begin{proof}[Proof of \Cref{thm:maxtothe3}]
The proof is by induction. The case $n=1$ is \Cref{prop:max5}. For the induction step, we simply combine the previous two claims.
Assume that $\MAX_{3^{n-1}+2} = T_{0,3^{n-1}+2}$ can be computed with $n$ hidden layers. 
By \Cref{claim:next-layer}, we know that
$T_{3^{n-1},2}$ can be computed with $n+1$ hidden layers. 
By \Cref{claim:compose-max}, we now know that
$T_{0,3^n+2}=\MAX_{3^n+2}$ can also be computed with $n+1$ hidden layers.
\end{proof}

\begin{proof}[Proof of \Cref{thm:main}]
    This follows from \Cref{thm:maxtothe3}, using the fact that every function in $\CPWL_n$ can be written as a linear combination of maxima of $n+1$ affine terms~\cite{wang2005generalization}.
\end{proof}

\begin{remark}
The inductive constructions described in the proofs of the Claims \ref{claim-max5} and \ref{claim:next-layer} imply that the weights of the obtained networks are fractions with denominators that are powers of two, in other words, binary fractions. Such fractions can also be represented as decimal fractions.
\end{remark}

\section{Geometric intuition}\label{sec:geometry}

Even though it is not necessary for the proof of our main theorem, in this section we describe the geometric intuition behind our constructions.
The geometric perspective may help to understand the constructions and may be useful in future research.
The connection to geometry is based on a well-known fundamental equivalence between (positively homogeneous and convex) $\CPWL$ functions and polytopes (see~\cite{zhang2018tropical,maclagan2021introduction,maragos2021tropical,hertrich2021towards,huchette2023deep,haase2023lower,brandenburg2025decomposition,hertrich2024neural,valerdi2024minimal}
and the references within), which we describe next.

\subsection{Polytopes, support functions, and neural networks}

A polytope is the convex hull of finitely many points in $\R^n$. We refer to \cite{gr2003unbaum} for standard terminology and theory of polytopes. Each polytope $P$ can be uniquely associated with its \emph{support function}
\[h_P(x) = \max \{ \ip{x}{p} : p \in P\}.\]
The support function $h_P$ is convex, positively homogeneous and $\CPWL$. Conversely, for each such function $f$ there exists a unique polytope $P$ with $f=h_P$, which is sometimes called the \emph{Newton polytope} of $f$, referring to the corresponding concept in tropical geometry.

The Newton polytope of $\MAX_n$ is the $(n-1)$-dimensional simplex $\Delta^{n-1}=\conv\{e_1,\dots,e_n\}$ in $\R^n$, where $e_i$ is the $i$'th standard unit vector.
Understanding the $\relu$ depth complexity of $\MAX_n$ translates to a question about generating $\Delta^{n-1}$ in the following model (see e.g.~\cite{hertrich2021towards,valerdi2024minimal}). For two sets $P,Q \subseteq \R^n$, their \emph{Minkowski sum} is defined as
\[P+Q = \{p+q : p \in P, q \in Q\}.\]
Their convex hull is denoted by
\[P*Q = \conv(P \cup Q).\]
Denote by $\cP_{0}$ the set of polytopes consisting of a single point (for a fixed dimension $n$). For $k \geq 0$, define
\[\cP_{k+1} = \Big\{ \sum_i P_i * Q_i : \forall i \ P_i,Q_i \in \cP_{k} \Big\},\]
where the sum over $i$ is finite. This leads to a computational complexity measure for polytopes. The complexity of a polytope $P$ is the minimum $k$ so that $P \in \cP_k$.

This complexity measure does not, however, capture $\relu$ depth complexity. The missing piece is formal Minkowski difference. A polytope $X$ is a formal Minkowski difference of $P$ and $Q$ if $X+P = Q$.
As the following lemma shows, this computational complexity measure is equivalent to $\relu$ depth.

\begin{lemma}[see \cite{hertrich2022facets}, Theorem 3.35]\label{lem:polytopereformulation}
For every polytope $X$,
the required number of hidden layers for representing $h_X$ is the minimum $k$ so that there are $P,Q \in \cP_{k}$ so that
\[X + P = Q.\]
\end{lemma}

We mostly care about $X = \Delta^{n-1}$.
In other words, we want to write $\Delta^{n-1}$ as a formal Minkowski differences of polytopes
in $\cP_k$ for the smallest $k$ possible. Understanding the $\relu$ depth complexity of the maximum function---and hence of all $\CPWL$ functions---boils down to understanding formal Minkowski differences. 
 
\subsection{Upper bounds on the depth via subdivisions}

By \Cref{lem:polytopereformulation}, we wish to write $\Delta^{n-1}$ as a formal Minkowski difference. We need to construct two polytopes $P,Q \in \cP_k$ for some small $k$, and show that $\Delta^{n-1} + P = Q$.
This task is not so easy because $P$ and $Q$ should be somehow ``coupled'' for this equation to hold.
We now explain how to replace the pair $(P,Q)$ by a subdivision of $\Delta^{n-1}$.
The advantage is that subdivisions are easier to comprehend and hence to find.

The simple fact we start with is that the identity map is what is called a \emph{valuation}: if $P$ and $Q$ are polytopes so that $P \cup Q$ is a polytope, then
\begin{align}
(P \cup Q) + (P \cap Q) = P+Q . \label{eqn:val}  
\end{align}
The deeper fact we use is an inclusion-exclusion extension of the valuation property.
It is called the \emph{full additivity} of the identity map; see for example 5.20 in~\cite{mcmullen1983valuations} and Chapter 6 in Schneider's book~\cite{schneider2013convex}.

\begin{lemma}[full additivity]
\label{lem:fullAd}
Let $X$ and $Q_1,\ldots,Q_m$ be polytopes in $\R^n$ so that $X = \bigcup_{i \in [m]} Q_i$.
Let $T^0$ be the collection of non-empty sets $S \subseteq [m]$ of even size.
Let $T^1$ be the collection of sets $S \subseteq [m]$ of odd size.
Then, 
\[X+\sum_{S \in T^0} \ Q_S = \sum_{S \in T^1} \ Q_S\]
where $Q_S:= \bigcap_{i \in S} Q_i$.
\end{lemma}

It is worth mentioning that \Cref{lem:fullAd} leads to a different proof of the important \Cref{eq:wangsun}
that was proved by
Wang and Sun~\cite{wang2005generalization}, compare also~\cite{hertrich2021towards}.
The explanation, in a nutshell, is as follows. 
Every polytope $P$ can be triangulated, that is, subdivided into simplices.
Then, by the \Cref{lem:fullAd},
the support function $h_P$ can be written as a difference between the support functions of simplices. 
In other words, $h_P$ is a signed-sum of $\MAX_{n+1}$ functions. 

The proof of \Cref{lem:fullAd} is quite deep. Like the inclusion-exclusion formula, it is tempting to prove it by induction. But as opposed to the inclusion-exclusion formula, if we remove a single polytope $Q_i$ then the relevant union may no longer be a polytope and the induction hypothesis can not be applied. 
We must significantly strengthen the lemma in order to prove it. The proof is obtained in several steps. 
First, it uses the equivalence between polytopes and their support functions. Second, it extends the support function to the lattice of finite unions of polytopes. This extension is topological and uses the Euler characteristic. Third, the Euler characteristic is shown to be fully additive. Finally, we can now deduce that the extended support function is fully additive, and hence the identity map on polytopes is fully additive (the identity map may not be fully additive on the lattice of finite unions of polytopes).

\Cref{lem:fullAd} is useful in combination with \Cref{lem:polytopereformulation}
because from a cover of $X=\Delta^{n-1}$ by polytopes $Q_1,\ldots,Q_m$,
we can write $\Delta^{n-1}$ as a formal Minkowski difference
between two polytopes that are as complex as the $Q_S$'s. 
This is captured by the following lemma.

\begin{lemma}\label{lemma:subdivision}
The required number of hidden layers for $\MAX_n$ is at most the minimum $k$ so that there are polytopes $Q_1,\ldots,Q_m$ so that $\Delta^{n-1} = \bigcup_{i \in [m]} Q_i$
where $Q_S \in \cP_{k}$ for all non-empty $S \subset [m]$.
\end{lemma}

Observe that the condition $Q_S \in \cP_{k}$ holds if $Q_1,\ldots,Q_m \in \cP_{k}$ 
and each $Q_S$ is a face of some $Q_i$ because $\cP_{k}$ is closed under taking faces. This observation suggests the following way of using \Cref{lemma:subdivision}.
Instead of writing $\Delta^{n-1}$ as a formal Minkowski difference, it suffices to subdivide $\Delta^{n-1}$ into $Q_1,\ldots,Q_m \in \cP_k$. If the subdivision is polyhedral, the intersections $Q_S$ are always faces, and such a subdivision automatically leads to writing $\Delta^{n-1}$ as $\Delta^{n-1}+P=Q$ for $P,Q \in \cP_k$. Note that $P$ and $Q$ themselves might be pretty complicated and  ``hard to find'' because they originate from the exponentially large alternating sums in \Cref{lem:fullAd}.

\subsection{Subdividing the four-simplex}

We now apply our previously developed machinery to explain the geometric intuition behind 
our construction of the two-hidden-layer $\relu$ neural network for $\MAX_5$.
Geometrically, we subdivide $\Delta^4$ to four pieces, each in $\cP_2$.
This subdivision leads to the two-hidden-layer $\relu$ network for $\MAX_5$ described in \Cref{eq-m}.

We start with the three-dimensional case $\Delta^3$.
Below we explain how to subdivide $\Delta^3$ into four pieces $Q_1,\ldots,Q_4$ so that each $Q_i$ is of the form
\[Q_i = p_i * Z_i\]
where $p_i \in \cP_0$ is a point and $Z_i \in \cP_{1}$ is a rhombus.
In words, each $Q_i$ is a rhombic pyramid.
This is a nice 3D puzzle, and an illustration of the solution is given in \Cref{fig:decomposition1} and \Cref{fig:decomposition2}.
This subdivision leads to a computation of $\MAX_4$ with two hidden layers, which is different than the ``trivial network''.

Let $x_1,\ldots,x_4$ denote the four points 
$$x_1=(-1,-1,-1), x_2=(1,1,-1), x_3=(-1,1,1), x_4=(1,-1,1).$$
The convex hull of the four points is a three-dimensional simplex (which is affinely equivalent to $\Delta^3$).
Denote by $x_{ij}$ the midpoints of $x_i$ and $x_j$.
The orthogonal projections of $x_1,\ldots,x_4$ to the plane $e_3^\perp$
are the four vertices of the square $[-1,1]^2$.
Cutting the simplex along the planes $e_1^\perp$ and $e_2^\perp$ partitions the simplex into four identical rhombic pyramids:
\begin{center}
\begin{tabular}{cc}
apex & base \\
\hline 
$x_{12}$ & $x_1 ,x_{14},x_{34}, x_{13}$ \\
$x_{12}$ & $x_2 , x_{23}, x_{34}, x_{24}$ \\
$x_{34}$ & $x_3 , x_{13}, x_{12}, x_{23}$ \\
$x_{34}$ & $x_4 , x_{14}, x_{12}, x_{24}$ \\
\end{tabular}
\end{center}
This is the representation we are after;
the apex is in $\cP_0$ and the base in $\cP_1$.

\begin{figure}\centering
\includegraphics[page=1]{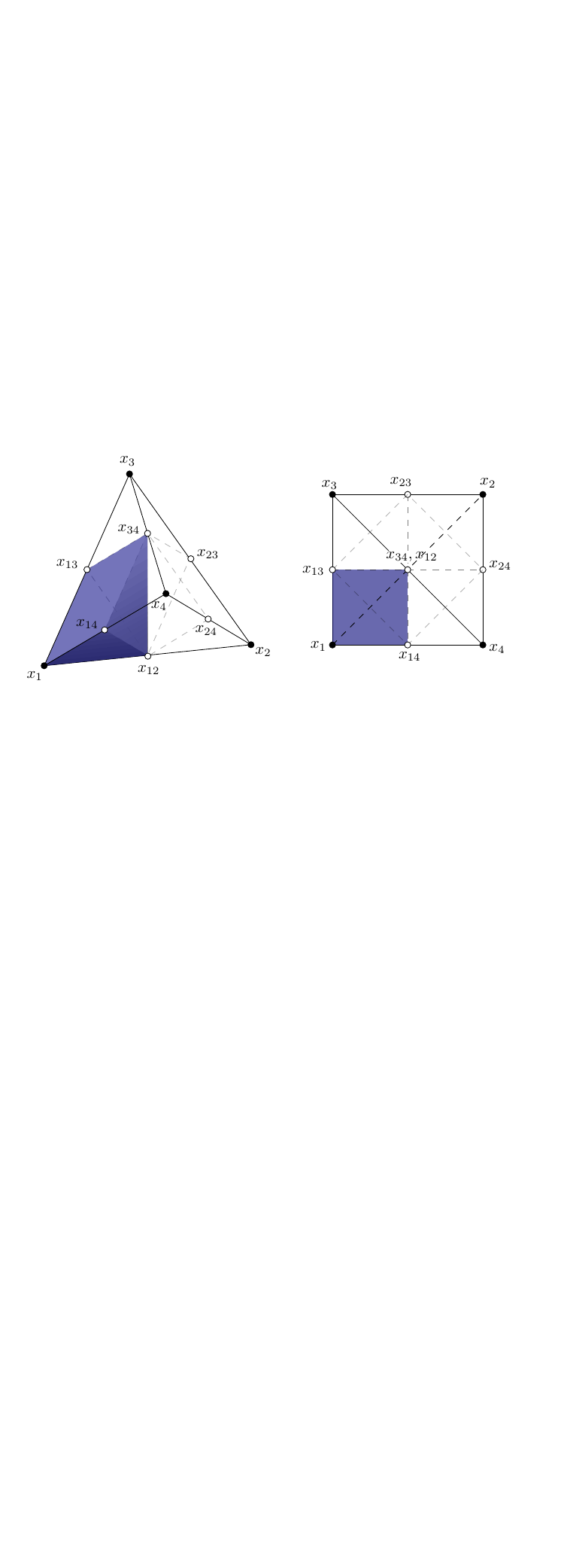}
\caption{Two ways to look at the three-simplex. The highlighted part is one of four pyramids of subdivision --- a rhombic pyramid with base $x_1,x_{14},x_{34},x_{13}$ and apex $x_{12}$. }
\label{fig:decomposition1}
\end{figure}

\begin{figure}\centering
\includegraphics[page=2]{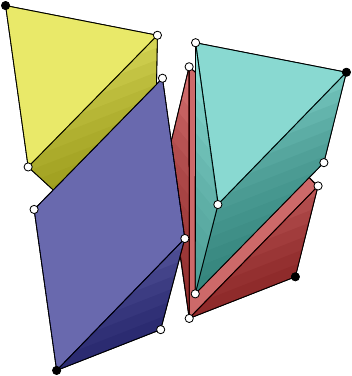}
\caption{Subdividing the three-simplex into four identical rhombic pyramids.}
\label{fig:decomposition2}
\end{figure}

We now explain how
the subdivision of $\Delta^3$ into rhombic pyramids leads to a subdivision of $\Delta^4$ into pieces in $\cP_2$. Let $X$ be a three-dimensional face of $\Delta^4$. So $X$ is a copy of $\Delta^3$. Subdivide $X$ into four rhombic pyramids of the form
$$p_i * Z_i.$$
Let $x_5$ be the vertex of $\Delta^4$ not in $X$. 
Because $\Delta^4 = x_5 * X$,
we see that the four pieces
$$x_5*(p_i * Z_i) = (x_5*p_i)*Z_i \in \cP_2$$
lead to the required subdivision of $\Delta^4$.

The four full-dimensional polytopes in $\cP_2$ in the subdivision are the Newton polytopes of the four terms $P_1(x)$, $P_2(x)$, $P_3(x)$, and $P_4(x)$ in \Cref{eq-m}, while the other terms in \Cref{eq-m} are the support functions of lower-dimensional faces of the four full-dimensional polytopes.

\subsection{Higher dimensions}

The geometric approach also allows to see why we can compute $\MAX_{3^n+2}$ in $n+1$ layers. We have seen that $\Delta^3$ can be subdivided into four pieces, each of which is a \emph{join}\footnote{A convex hull of two polytopes that are embedded in non-intersecting affine spaces.} 
of a point $p_i$ and a parallelogram $Z_i$. The simplex $\Delta^3$ is the join of a point and a triangle.
In other words, we replaced a triangle (which is in $\cP_2$) with a parallelogram (which is in $\cP_1$).
More abstractly, we replaced three points (the triangle) by a single parallelogram, which is the geometric intuition behind \Cref{claim-max5}.
This suggests that each layer can reduce ``the number of objects to be joined'' by a factor of three, which should then allow $\Delta^n$ to be computed in roughly $\log_3(n)$ layers.
A formal implementation of this proof strategy can be done, but it requires a non-trivial reasoning on the interactions between joins
and subdivisions. The algebraic proof presented in \Cref{sec:newupperbound}
``hides'' this complex interaction in the vector space structure of $\mathcal{T}_{a,b}$.

\section{Conclusions}
Summarizing, we provide the first improvement of the upper bound on the required depth to represent CPWL functions, after the initial representation result by Arora, Basu, Mianjy, and Mukherjee~\cite{arora2018understanding}. Thereby, we break the $\log_2(n)$-barrier that was conjectured to be optimal~\cite{hertrich2021towards} and proven to be tight in many special cases~\cite{haase2023lower,valerdi2024minimal,grillo2025depth}. Consequently, our result has the potential to shift the focus of future work from lower bounds to upper bounds, from impossibility results to constructions. One interpretation of our result is that neural networks can compute exact maxima more efficiently than in a pairwise manner. Besides the result itself, a key contribution of our paper is that new upper bounds can be proven via subdividing the simplex into polytopes with known depth complexity, which is captured by \Cref{lemma:subdivision}.

Curiously, while we were able to give a positive answer to the previously open question of whether the maximum of five numbers can be computed with two hidden layers, the same question remains open for the maximum of six numbers. The lower bound by Averkov, Hojny, and Merkert~\cite{averkov2025expressiveness} for neural networks with rational weights, however, tells us that further improving the upper bound would require rational weights with large denominators (or even irrational weights) in the neural network. Geometrically, when applying our subdivision approach, this would correspond to subdivisions of the simplex into even smaller cells. Overall, the question of whether constant depth, perhaps even two hidden layers, suffices to represent all CPWL functions remains an intriguing mystery.

\section*{Acknowledgments}

We thank Daniel Reichman for valuable discussions around $\relu$ networks. C.H.\ thanks Gennadiy Averkov for valuable discussions around \Cref{lem:fullAd} during the Cargese Workshop on Combinatorial Optimization 2024.
E.B., F.B., and A.Y.\ are supported by a DNRF Chair grant.
E.B., F.B., J.S., and A.Y.\ are part of BARC, Basic Algorithms Research Copenhagen, supported by the VILLUM Foundation grant 54451.

\bibliographystyle{alphaurl}
\bibliography{mr.bib}

\end{document}